\documentclass{article}

\usepackage{times}
\usepackage{graphicx} % more modern
\usepackage{subfigure} 

\usepackage{natbib}

\usepackage{algorithm}
\usepackage{algorithmic}

\usepackage{amsmath}
\usepackage{amssymb}
\usepackage{amsthm}
\usepackage{thmtools,thm-restate}

\usepackage{hyperref}

\usepackage[accepted]{icml2016}

\icmltitlerunning{Easy Monotonic Policy Iteration}

\begin{document} 

%
% Definitions and macros
%

%\setlength{\marginparwidth}{1.2in}
%\let\oldmarginpar\marginpar
%\renewcommand\marginpar[1]{\-\oldmarginpar[\raggedleft\footnotesize #1]%
%{\raggedright\footnotesize #1}}

%\renewcommand{\indexspace}{\rule{0cm}{.4cm}}
%% end example/remark
%\newcommand{\eex}{\ifmmode\sq\else{\unskip\nobreak\hfil
%  \penalty50\hskip1em\null\nobreak\hfil$\Diamond$
%  \parfillskip=0pt\finalhyphendemerits=0\endgraf}\fi{}}
%\newcommand{\erem}{\ifmmode\sq\else{\unskip\nobreak\hfil
%  \penalty50\hskip1em\null\nobreak\hfil$\star$
%  \parfillskip=0pt\finalhyphendemerits=0\endgraf}\fi{}}
%\newcommand{\eobs}{\ifmmode\sq\else{\unskip\nobreak\hfil
%  \penalty50\hskip1em\null\nobreak\hfil$\vartriangleleft$
%  \parfillskip=0pt\finalhyphendemerits=0\endgraf}\fi{}}

% VET - characters - lowercase
\newcommand{\avet}{{\mathbf  a}}
\newcommand{\bvet}{{\mathbf  b}}
\newcommand{\cvet}{{\mathbf  c}}
\newcommand{\dvet}{{\mathbf  d}}
\newcommand{\evet}{{\mathbf  e}}
\newcommand{\fvet}{{\mathbf  f}}
\newcommand{\gvet}{{\mathbf  g}}
\newcommand{\hvet}{{\mathbf  h}}
\newcommand{\ivet}{{\mathbf  i}}
\newcommand{\jvet}{{\mathbf  j}}
\newcommand{\kvet}{{\mathbf  k}}
\newcommand{\lvet}{{\mathbf  l}}
\newcommand{\mvet}{{\mathbf  m}}
\newcommand{\nvet}{{\mathbf  n}}
\newcommand{\ovet}{{\mathbf  o}}
\newcommand{\pvet}{{\mathbf  p}}
\newcommand{\qvet}{{\mathbf  q}}
\newcommand{\rvet}{{\mathbf  r}}
\newcommand{\svet}{{\mathbf  s}}
\newcommand{\tvet}{{\mathbf  t}}
\newcommand{\uvet}{{\mathbf  u}}
\newcommand{\vvet}{{\mathbf  v}}
\newcommand{\xvet}{{\mathbf  x}}
\newcommand{\yvet}{{\mathbf  y}}
\newcommand{\zvet}{{\mathbf  z}}
\newcommand{\wvet}{{\mathbf  w}}

% VET - characters - uppercase
\newcommand{\Avet}{{\mathbf  A}}
\newcommand{\Bvet}{{\mathbf  B}}
\newcommand{\Cvet}{{\mathbf  C}}
\newcommand{\Dvet}{{\mathbf  D}}
\newcommand{\Evet}{{\mathbf  E}}
\newcommand{\Fvet}{{\mathbf  F}}
\newcommand{\Gvet}{{\mathbf  G}}
\newcommand{\Hvet}{{\mathbf  H}}
\newcommand{\Ivet}{{\mathbf  I}}
\newcommand{\Jvet}{{\mathbf  J}}
\newcommand{\Kvet}{{\mathbf  K}}
\newcommand{\Lvet}{{\mathbf  L}}
\newcommand{\Mvet}{{\mathbf  M}}
\newcommand{\Nvet}{{\mathbf  N}}
\newcommand{\Ovet}{{\mathbf  O}}
\newcommand{\Pvet}{{\mathbf  P}}
\newcommand{\Qvet}{{\mathbf  Q}}
\newcommand{\Rvet}{{\mathbf  R}}
\newcommand{\Svet}{{\mathbf  S}}
\newcommand{\Tvet}{{\mathbf  T}}
\newcommand{\Uvet}{{\mathbf  U}}
\newcommand{\Xvet}{{\mathbf  X}}
\newcommand{\Yvet}{{\mathbf  Y}}
\newcommand{\Vvet}{{\mathbf  V}}
\newcommand{\Wvet}{{\mathbf  W}}
\newcommand{\Zvet}{{\mathbf  Z}}

\newcommand{\Deltavet}{\mathbf  \Delta}
\newcommand{\Lambdavet}{{\mathbf  \Lambda}}
\newcommand{\Sigmavet}{\mathbf  \Sigma}
\newcommand{\Thetavet}{{\mathbf  \Theta}}

% Special characters:
\newcommand{\s}{ {\sigma} }

\newcommand{\e}{{\mathrm e}}
\newcommand{\jm}{{\mathrm j}}
\newcommand{\E}{{\mathrm E}}
\newcommand{\Ex}{{\mathbb E}}
\renewcommand{\d}{{\mathrm d}}
\newcommand{\dt}{{\mathrm d}t}
\newcommand{\X}{ {\mathcal X} }
\newcommand{\Y}{ {\mathcal Y} }
\newcommand{\Z}{ {\mathcal Z} }

\newcommand{\calA}{{\mathcal A}}
\newcommand{\calB}{{\mathcal B}}
\newcommand{\calC}{{\mathcal C}}
\newcommand{\calD}{{\mathcal D}}
\newcommand{\calE}{{\mathcal E}}
\newcommand{\calF}{{\mathcal F}}
\newcommand{\calG}{{\mathcal G}}
\newcommand{\calH}{{\mathcal H}}
\newcommand{\calI}{{\mathcal I}}
\newcommand{\calJ}{{\mathcal J}}
\newcommand{\calK}{{\mathcal K}}
\newcommand{\calL}{{\mathcal L}}
\newcommand{\calM}{{\mathcal M}}
\newcommand{\calN}{{\mathcal N}}
\newcommand{\calO}{{\mathcal O}}
\newcommand{\calP}{{\mathcal P}}
\newcommand{\calQ}{{\mathcal Q}}
\newcommand{\calR}{{\mathcal R}}
\newcommand{\calS}{{\mathcal S}}
\newcommand{\calT}{{\mathcal T}}
\newcommand{\calU}{{\mathcal U}}
\newcommand{\calV}{{\mathcal V}}
\newcommand{\calX}{{\mathcal X}}
\newcommand{\calY}{{\mathcal Y}}
\newcommand{\calW}{{\mathcal W}}
\newcommand{\calZ}{{\mathcal Z}}
\newcommand{\qtil}{{\tilde{q}}}
\newcommand{\td}{{\tilde{\delta}}}

\newcommand{\vect}[1]{ {\mbox{\rm vec}(#1)} }

% Macro comandi:

\newcommand{\Atil}{\tilde{A}}
\newcommand{\Zhat}{\hat{Z}}
\newcommand{\Hbar}{\bar{H}}
\newcommand{\Dhat}{\hat{D}}
\newcommand{\dhat}{\hat{d}}

\newcommand{\rhat}{\hat{r}}
\newcommand{\xhat}{\hat{x}}
\newcommand{\yhat}{\hat{y}}
\newcommand{\zhat}{\hat{z}}
\newcommand{\xbar}{\bar{x}}
\newcommand{\ubar}{\bar{u}}
\newcommand{\ybar}{\bar{y}}
\newcommand{\zbar}{\bar{z}}
\newcommand{\pdot}{\dot{p}}
\newcommand{\pddot}{\ddot{p}}
\newcommand{\pbar}{\bar{p}}
\newcommand{\qdot}{\dot{q}}
\newcommand{\qddot}{\ddot{q}}
\newcommand{\qbar}{\bar{q}}
\newcommand{\xdot}{\dot{x}}
\newcommand{\ydot}{\dot{y}}
\newcommand{\zdot}{\dot{z}}
\newcommand{\yddot}{\ddot{y}}
\newcommand{\thdot}{\dot{\theta}}
\newcommand{\thddot}{\ddot{\theta}}
\newcommand{\util}{{\tilde{u}}}
\newcommand{\xtil}{{\tilde{x}}}
\newcommand{\ytil}{{\tilde{y}}}
\newcommand{\lam}{\lambda}
\newcommand{\lamax}{\lambda\ped{max}}
\newcommand{\lamin}{\lambda\ped{min}}
\newcommand{\adj}{ {\mbox{\rm adj}\;} }
\newcommand{\sign}{\mbox {\rm sgn}}
\newcommand{\spn}{\mbox {\rm span}}
\newcommand{\barJ}{\bar{J}}
\newcommand{\dom}{\mathop {\mathrm {dom}}}
\newcommand{\card}{\mathop{\mathrm{card}}}
\newcommand{\subt}{\mathop{\mathrm{s.t.}}}

\newcommand{\epi}{\mathop{\mathrm{epi}}}
\newcommand{\env}{\mathop{\mathrm{env}}}
\newcommand{\chull}{\mathop{\mathrm{co}}}
\newcommand{\graph}{\mathop{\mathrm{graph}}}
\newcommand{\prox}[1]{\mathop{\mathrm{prox}_{#1}}}
\newcommand{\sthr}[1]{\mathop{\mathrm{sthr}_{#1}}}

\def\hardsection{$\spadesuit\;$}

%%%% Fields and Groups
\newcommand{\Real}[1]{ { {\mathbb R}^{#1} } }
\newcommand{\Realp}[1]{ { {\mathbb R}_{+}^{#1} } }
\newcommand{\Realpp}[1]{ { {\mathbb R}_{++}^{#1} } }
\newcommand{\Complex}[1]{ { {\mathbb C}^{#1} } }
\newcommand{\Imag}[1]{ { {\mathbb I}^{#1} } }
\newcommand{\Field}[1]{ {\mathbb F}^{#1} }
\newcommand{\F}{ {\mathbb F}}
\newcommand{\Orth}[1]{ { {\calG_{\calO}^{#1}} } }
\newcommand{\Unit}[1]{ { {\calG_{\calU}^{#1}} } }
\newcommand{\Sym}[1]{ { {\mathbb S}^{#1} } }
\newcommand{\Symp}[1]{ { {\mathbb S}_{+}^{#1} } }
\newcommand{\Sympp}[1]{ { {\mathbb S}_{++}^{#1} } }
\newcommand{\Herm}[1]{ { {\mathbb H}^{#1} } }
\newcommand{\Skew}[1]{ { {\mathbb S\mathbb K}^{#1} } }
\newcommand{\Skherm}[1]{ { {\mathbb H\mathbb K}^{#1} } }
% manifolds (in matrices)
\newcommand{\Rman}[1]{ { {\mathcal R}^{#1} } } 
\newcommand{\Cman}[1]{ { {\mathcal C}^{#1} } }
\newcommand{\Hinf}[1]{ {  {\mathcal H}_\infty^{#1} } }
\newcommand{\RHinf}[1]{ { {\mathcal RH}_\infty^{#1} } }
\newcommand{\Htwo}[1]{ {  {\mathcal H}_2^{#1} } }
\newcommand{\RHtwo}[1]{ { {\mathcal RH}_2^{#1} } }

\newcommand{\dist}[1]{{\mathrm{dist}}{\left( #1 \right)}}
\newcommand{\diff}[2]{ \frac{\d {#1}}{\d {#2}}  }
\newcommand{\diffp}[2]{ \frac{\partial {#1}}{\partial {#2}}  }
\newcommand{\diffqd}[2]{ \frac{\d^2 {#1}}{\d {#2}^2}  }
\newcommand{\diffq}[2]{ \frac{\d^2 {#1}}{\d {#2}}  }
\newcommand{\diffqq}[3]{ \frac{\d^2 {#1}}{ \d {#2} \d {#3}  }}
\newcommand{\diffpq}[2]{ \frac{\partial^2 {#1}}{\partial {#2}^2}  }
\newcommand{\difftq}[3]{ \frac{\partial^2 {#1}}{\partial {#2}\partial {#3}}  }
\newcommand{\diffi}[3]{ \frac{\d^{#3} {#1}}{\d {#2}^{#3}}  }
\newcommand{\diffpi}[3]{ \frac{\partial^{#3} {#1}}{\partial {#2}^{#3}}  }
\newcommand{\binomial}[2]{\scriptsize{\left(\!\! \ba{c} #1 \\ #2 \ea \!\! \right)} }
\newcommand{\comb}[2]{{\left(\!\!\! \ba{c} #1 \\ #2 \ea \!\!\! \right)} }

\newcommand{\simax}{{\sigma_{\mathrm{max}}}}
\newcommand{\simin}{{\sigma_{\mathrm{min}}}}
\newcommand{\prob}{{\mbox{\rm Prob}}}
\newcommand{\var}{{\mbox{\rm var}}}
\newcommand{\sint}{{\mbox{\rm int}\,}} %set interior
\newcommand{\relint}{{\mbox{\rm relint}\,}} %set interior
\newcommand{\ns}{{\mbox{\tt ns}}}

\newcommand{\rank}{\mathop{\mathrm{rank}}\nolimits}
\newcommand{\range}{\mathop{\mathcal{R}}\nolimits}
\newcommand{\nulsp}{\mathop{\mathcal{N}}\nolimits}
\newcommand{\diagop}{\mathop{\mathrm{diag}}\nolimits}
\newcommand{\Var}{\mathop{\mathrm{var}}\nolimits}
\newcommand{\tr}{\mathop{\mathrm{trace}}\nolimits}
\newcommand{\sinc}{\mathop{\mathrm{sinc}}\nolimits}

%%%% Real and Imaginary
\newcommand{\pre}[1]{ { {\mathop{\mathrm{Re}}}  \left({#1}\right)} }
\newcommand{\pim}[1]{ { {\mathop{\mathrm{Im}}}  ({#1})} }
\newcommand{\rp}{ ^{\Real{}} }
\newcommand{\ip}{ ^{\Imag{}} }

%%%% Various
\newcommand{\one}{{\mathbf  1}}
\newcommand{\dss}{\displaystyle}
\newcommand{\inv}{^{-1}}
\newcommand{\pinv}{^{\dagger}}
\newcommand{\diag}[1]{\mathrm{diag}\left({#1}\right)}
\newcommand{\blockdiag}[1]{\mbox{\rm bdiag}\left({#1}\right)}
\newcommand{\tran}{^{\top}}
\newcommand{\inner}[1]{\langle {#1} \rangle}
\newcommand{\ped}[1]{_{\mathrm{#1}}}
\newcommand{\ap}[1]{^{\mathrm{#1}}}

\newcommand{\blu}[1]{\textcolor{blue}{#1}}
\newcommand{\red}[1]{\textcolor{red}{#1}}
\newcommand{\green}[1]{\textcolor{green}{#1}}
\newcommand{\cyan}[1]{\textcolor{cyan}{#1}}
%\newcommand{\comment}[1]{\vspace{.1cm} \blu{#1} \vspace{.1cm}}

%%%% Commands
\newcommand{\beq}{\begin{equation}}
\newcommand{\eeq}{\end{equation}}
\newcommand{\bea}{\begin{eqnarray}}
\newcommand{\eea}{\end{eqnarray}}
\newcommand{\beas}{\begin{eqnarray*}}
\newcommand{\eeas}{\end{eqnarray*}}
\newcommand{\ba}{\begin{array}}
\newcommand{\ea}{\end{array}}
\newcommand{\bit}{\begin{itemize}}
\newcommand{\eit}{\end{itemize}}
\newcommand{\ben}{\begin{enumerate}}
\newcommand{\een}{\end{enumerate}}
\newcommand{\bde}{\begin{description}}
\newcommand{\ede}{\end{description}}
\newcommand{\bsp}{\begin{split}}
\newcommand{\esp}{\end{split}}

%% Environments
\newtheorem{corollary}{Corollary}
\newtheorem{theorem}{Theorem}
\newtheorem{exercise}{Exercise}
\newtheorem{solution}{Solution}
\newtheorem{assumption}{Assumption}
\newtheorem{definition}{Definition}
\newtheorem{proposition}{Proposition}
\newtheorem{lemma}{Lemma}
\newtheorem{fact}{Fact}
\theoremstyle{remark}
\newtheorem*{remark}{Remark}

%%% margin stuff
%% example in the margin
%\newcommand{\marginex}[1]{
%\marginnote{\refstepcounter{examplectr}{\bfseries\textsf{Example \theexamplectr.}} 
%#1
%}}
%
%% remark in the margin
%\newcommand{\marginrmk}[1]{
%\marginnote{\refstepcounter{remarkctr}{\bfseries\textsf{Remark \theremarkctr.}} 
%#1
%}}
%
%% algorithm in the margin
%\newcommand{\marginalg}[1]{
%\marginnote{\refstepcounter{algorithmctr}{\bfseries\textsf{Algorithm \thealgorithmctr.}} 
%#1
%}}

%% misc
% these commands are to make things compile
\def\nocolon{}

% Prints the month name (e.g., January) and the year (e.g., 2008)
\newcommand{\monthyear}{%
  \ifcase\month\or January\or February\or March\or April\or May\or June\or
  July\or August\or September\or October\or November\or
  December\fi\space\number\year
}

% Prints an epigraph and speaker in sans serif, all-caps type.
\newcommand{\openepigraph}[2]{%
  %\sffamily\fontsize{14}{16}\selectfont
  \begin{fullwidth}
  \sffamily\large
  \begin{doublespace}
  \noindent\allcaps{#1}\\% epigraph
  \noindent\allcaps{#2}% author
  \end{doublespace}
  \end{fullwidth}
}

% Inserts a blank page
\newcommand{\blankpage}{\newpage\hbox{}\thispagestyle{empty}\newpage}

\twocolumn[
\icmltitle{Easy Monotonic Policy Iteration}

\icmlauthor{Joshua Achiam}{jachiam@berkeley.edu}
\icmladdress{UC Berkeley}
%\icmlauthor{Your CoAuthor's Name}{email@coauthordomain.edu}
%\icmladdress{Their Fantastic Institute,
%            27182 Exp St., Toronto, ON M6H 2T1 CANADA}

\icmlkeywords{reinforcement learning, policy improvement bounds, policy iteration}

\vskip 0.3in
]

\begin{abstract} 
A key problem in reinforcement learning for control with general function approximators (such as deep neural networks and other nonlinear functions) is that, for many algorithms employed in practice, updates to the policy or $Q$-function may fail to improve performance---or worse, actually cause the policy performance to degrade. Prior work has addressed this for policy iteration by deriving tight policy improvement bounds; by optimizing the lower bound on policy improvement, a better policy is guaranteed. However, existing approaches suffer from bounds that are hard to optimize in practice because they include sup norm terms which cannot be efficiently estimated or differentiated. In this work, we derive a better policy improvement bound where the sup norm of the policy divergence has been replaced with an average divergence; this leads to an algorithm, Easy Monotonic Policy Iteration, that generates sequences of policies with guaranteed non-decreasing returns and is easy to implement in a sample-based framework. 
\end{abstract} 

\section{Introduction}

Following the success of the Deep Q-Network (DQN) approach \cite{Mnih2013}, there has been a surge of interest in using reinforcement learning for control with nonlinear function approximators, particularly with deep neural networks; in these methods, the policy or the Q-function is represented by the approximator. Examples include variants on DQN, such as Double-DQN \cite{VanHasselt2015} and Deep Recurrent Q-Learning \cite{Hausknecht2015}, as well as methods that mix neural network policies and neural network value functions, such as the asynchronous advantage actor-critic algorithm \cite{Mnih2016}. However, despite empirical successes, there are few algorithms that come with theoretical guarantees for continued policy improvement when training policies represented by arbitrary nonlinear function approximators.  

One approach, which serves as the inspiration for this work, seeks to maximize a lower bound on policy performance to guarantee an improvement. This method has its roots in conservative policy iteration (CPI) \cite{Kakade2002}, and was extended separately by Pirotta et~al. \yrcite{Pirotta2013} and Schulman et~al. \yrcite{Schulman2006}. Both Pirotta et~al. and Schulman et~al. derived similar policy improvement bounds consisting of two parts: an expected advantage of the new policy with respect to the old policy, and a penalty on a divergence between the new policy and the old policy. The divergence penalty, in both cases, is quite steep: it involves the maximum policy divergence over all states. This makes it particularly difficult to apply the bounds in the usual situations where function approximation is desirable: domains where the model---and hence the total set of states---is unknown (in which case it is not possible to evaluate the bound) and/or where the state space is large (in which case the bound may be unnecessarily conservative). Pirotta et~al. developed algorithms that primarily apply to the case where the model is known or where approximation is present only in the advantage estimation, and so did not address this issue. Schulman et~al. addressed the issue by proposing to solve an approximate form of the problem, where the maximum divergence penalty is replaced by a trust-region constraint on the average divergence; this algorithm is called trust region policy optimization (TRPO). They found TRPO worked quite well in a number of domains, successfully optimizing neural network policies to play Atari from raw pixels and to control a simulated robot in locomotion tasks. 

In this work, we derive a new policy improvement bound where the penalty on policy divergence goes as an average, rather than the maximum, divergence. This allows us to propose Easy Monotonic Policy Iteration (EMPI), an algorithm that exploits the bound to generate sequences of policies with guaranteed non-decreasing returns, which is easy to implement in a sample-based framework. It also enables us to give a new theoretical justification for TRPO: we are able to show that each iteration of TRPO has a worst-case degradation of policy performance which depends on a hyperparameter of the algorithm. Our contributions at present are entirely theoretical, but empirical results from testing EMPI will appear in a future version of this work.

%We also consider generalizations of this bound that allow for wider applicability: particularly, we develop an analogous bound for deterministic policies that are greedy with respect to Q-function approximators. The bound for Q-functions is not differentiable with respect to the parameters of the approximator because argmax policies have zero derivative almost everywhere; to deal with this, we consider relaxations inspired by entropy-regularization to obtain a differentiable bound, and attempt to map between the relaxation and the original problem. This results in an algorithm which is closely related to DQN, and empirical results suggest that DQN is, indeed, an approximation of this algorithm.  

\section{Preliminaries}

A Markov decision process is a tuple, ($S,A,R,P,\mu$), where $S$ is the set of states, $A$ is the set of actions, $R : S \times A \times S \to \Real{}$ is the reward function, $P : S \times A \times S \to [0,1]$ is the transition probability function (where $P(s'|s,a)$ is the probability of transitioning to state $s'$ given that the previous state was $s$ and the agent took action $a$ in $s$), and $\mu : S \to [0,1]$ is the starting state distribution. A policy $\pi: S \times A \to [0,1]$ is a distribution over actions per state, with $\pi(a|s)$ the probability of selecting $a$ in state $s$. We consider the problem of picking a policy $\pi$ that maximizes the expected infinite-horizon discounted total reward,
\begin{equation*}
J(\pi) \doteq \underset{\tau \sim \pi}{\E} \left[ \sum_{t=0}^{\infty} \gamma^t R(s_t,a_t,s_{t+1})\right],
\end{equation*}
where $\gamma \in [0,1)$ is the discount factor, $\tau$ denotes a trajectory ($\tau = (s_0, a_0, s_1, ...)$), and $\tau \sim \pi$ is shorthand for indicating that trajectories are drawn from distributions induced by $\pi$: $s_0 \sim \mu$, $a_t \sim \pi(\cdot|s_t)$, $s_{t+1} \sim P(\cdot | s_t, a_t)$. 

We define the on-policy value function, $V^{\pi} : S \to \Real{}$, action-value function $Q^{\pi} : S \times A \to \Real{}$, and advantage function $A^{\pi} : S \times A \to \Real{}$ in the usual way:
\begin{eqnarray*}
V^{\pi} (s) &\doteq& \underset{a_0, s_1,...}{\E} \left[ \left.\sum_{t=0}^{\infty} \gamma^t R_t \right| s_0 = s\right], \\
Q^{\pi} (s,a) &\doteq& \underset{s_1, a_1, ...}{\E} \left[ \left.\sum_{t=0}^{\infty} \gamma^t R_t \right| s_0 = s, a_0 = a\right], \\
A^{\pi} (s,a) &\doteq& Q^{\pi} (s,a) - V^{\pi} (s),
\end{eqnarray*}
where $R_t = R(s_t,a_t,s_{t+1})$. The $Q$ and $V$ functions are connected by
\begin{equation*}
Q^{\pi} (s,a) = \underset{s' \sim P(\cdot|s,a)}{\E} \left[ \left. R(s,a,s') + \gamma V^{\pi} (s') \right| s, a \right].
\end{equation*}

Our analysis will make extensive use of the discounted future state distribution, $d^{\pi}$, which is defined as
\begin{equation*}
d^{\pi} (s) = (1-\gamma) \sum_{t=0}^{\infty} \gamma^t P(s_t = s | \pi).
\end{equation*}
It allows us to express the expected discounted total reward  compactly as
\begin{equation}
J(\pi) = \frac{1}{1-\gamma} \underset{\begin{subarray}{c} s \sim d^{\pi} \\ a \sim \pi \\ s' \sim P\end{subarray}}{\E} \left[ R(s,a,s') \right], \label{jpi}
\end{equation}
where by $a \sim \pi$, we mean $a \sim \pi(\cdot|s)$, and by $s' \sim P$, we mean $s' \sim P(\cdot|s,a)$. We drop the explicit notation for the sake of reducing clutter, but it should be clear from context that $a$ and $s'$ depend on $s$. 

Next, we will examine some useful properties of $d^{\pi}$ that become apparent in vector form for finite state spaces. Let $p^t_{\pi} \in \Real{|S|}$ denote the vector with components $p^t_{\pi} (s) = P(s_t = s | \pi)$, and let $P_{\pi} \in \Real{|S|\times|S|}$ denote the transition matrix with components $P_{\pi} (s'|s) = \int da P(s'|s,a) \pi(a|s)$; then $p^t_{\pi} = P_{\pi} p^{t-1}_{\pi} = P^t_{\pi} \mu$ and 
\begin{eqnarray}
d^{\pi} &=& (1-\gamma) \sum_{t=0}^{\infty} (\gamma P_{\pi})^t \mu \nonumber \\
&=& (1-\gamma) (I - \gamma P_{\pi})^{-1} \mu. \label{dpi}
\end{eqnarray} 

This formulation helps us easily obtain the following lemma.

\begin{lemma} For any function $f : S \to \Real{}$ and any policy $\pi$,
\begin{equation}
(1-\gamma) \underset{s \sim \mu}{\E}\left[f(s)\right] + \underset{\begin{subarray}{c} s \sim d^{\pi} \\ a \sim \pi \\ s' \sim P\end{subarray}}{\E} \left[ \gamma f(s') \right] - \underset{s \sim d^{\pi}}{\E}\left[f(s)\right] = 0. 
\end{equation}
\end{lemma}

\begin{proof} Multiply both sides of (\ref{dpi}) by $(I - \gamma P_{\pi})$ and take the inner product with the vector $f \in \Real{|S|}$. 
\end{proof}

Combining this with (\ref{jpi}), we obtain the following, for any function $f$ and any policy $\pi$:
\begin{equation}
\begin{aligned}
J(\pi) = &\underset{s \sim \mu}{\E}[f(s)] \\
&+ \frac{1}{1-\gamma} \underset{\begin{subarray}{c} s \sim d^{\pi} \\ a \sim \pi \\ s' \sim P\end{subarray}}{\E} \left[R(s,a,s') + \gamma f(s') - f(s)\right]. \label{jpif}
\end{aligned}
\end{equation}

This identity is nice for two reasons. First: if we pick $f$ to be an approximator of the value function $V^{\pi}$, then (\ref{jpif}) relates the true discounted return of the policy ($J(\pi)$) to the estimate of the policy return ($\E_{s\sim \mu}[f(s)]$) and to the on-policy average TD-error of the approximator; this is aesthetically satisfying. Second: it shows that reward-shaping by $\gamma f(s') - f(s)$ has the effect of translating the total discounted return by $\E_{s\sim \mu} [f(s)]$, a fixed constant independent of policy; this illustrates the finding of Ng. et~al. \yrcite{Ng1999} that reward shaping by $\gamma f(s') + f(s)$ does not change the optimal policy.

It is also helpful to introduce an identity for the vector difference of the discounted future state visitation distributions on two different policies, $\pi'$ and $\pi$. Define the matrices $G \doteq (I - \gamma P_{\pi})^{-1}$, $\bar{G} \doteq (I - \gamma P_{\pi'})^{-1}$, and $\Delta = P_{\pi'} - P_{\pi}$. Then:
\begin{eqnarray*}
G^{-1} - \bar{G}^{-1} &=& (I - \gamma P_{\pi}) - (I - \gamma P_{\pi'}) \\
&=& \gamma \Delta;
\end{eqnarray*} 
left-multiplying by $G$ and right-multiplying by $\bar{G}$, we obtain
\begin{equation*}
\bar{G}  - G = \gamma \bar{G} \Delta G. 
\end{equation*}

Thus
\begin{eqnarray}
d^{\pi'} - d^{\pi} &=& (1-\gamma) \left(\bar{G} - G\right) \mu \nonumber \\
&=& \gamma (1-\gamma) \bar{G} \Delta G \mu \nonumber \\
&=& \gamma \bar{G} \Delta d^{\pi}. \label{dpidiff}
\end{eqnarray}

For simplicity in what follows, we will only consider MDPs with finite state and action spaces, although our attention is on MDPs that are too large for tabular methods. 

\section{Main Results}

In this section, we will derive and present the new policy improvement bound. We will begin with a lemma:

\begin{lemma}\label{policybound0} For any function $f: S \to \Real{}$ and any policies $\pi'$ and $\pi$, define
\begin{equation}
\begin{aligned}
&L_{\pi,f} (\pi') \doteq \\
& \underset{\begin{subarray}{c} s \sim d^{\pi} \\ a \sim \pi \\ s' \sim P\end{subarray}}{\E} \left[ \left(\frac{\pi'(a|s)}{\pi(a|s)} - 1 \right) \left(R(s,a,s') + \gamma f(s') - f(s) \right)\right],
\end{aligned} \label{surrogate}
\end{equation}
and $\epsilon_f^{\pi'} \doteq \max_s \left| \E_{a \sim \pi', s'\sim P} [R(s,a,s') + \gamma f(s') - f(s)] \right|$. Then the following bound holds:
\begin{equation}
J(\pi') - J(\pi) \geq \frac{1}{1-\gamma}\left(L_{\pi,f} (\pi') - 2\epsilon_f^{\pi'} D_{TV} (d^{\pi'} || d^{\pi})\right), \label{bound0}
\end{equation}
where $D_{TV}$ is the total variational divergence. Furthermore, the bound is tight (when $\pi' = \pi$, the LHS and RHS are identically zero). 
\end{lemma}

\begin{proof} First, for notational convenience, let $\delta_f (s,a,s') \doteq R(s,a,s') + \gamma f(s') - f(s)$. (The choice of $\delta$ to denote this quantity is intentionally suggestive---this bears a strong resemblance to a TD-error.) By (\ref{jpif}), we obtain the identity
\begin{equation*}
\begin{aligned}
&J(\pi') - J(\pi) \\
&= \frac{1}{1-\gamma} \left(\underset{\begin{subarray}{c} s \sim d^{\pi'} \\ a \sim \pi' \\ s' \sim P\end{subarray}}{\E} \left[ \delta_f (s,a,s') \right] - \underset{\begin{subarray}{c} s \sim d^{\pi} \\ a \sim \pi \\ s' \sim P\end{subarray}}{\E} \left[ \delta_f(s,a,s') \right].\right)
\end{aligned}
\end{equation*}

Now, we restrict our attention to the first term in this equation. Let $\bar{\delta}_f^{\pi'} \in \Real{|S|}$ denote the vector of components $\bar{\delta}_f^{\pi'} (s) = \E_{a \sim \pi', s' \sim P} [\delta_f(s,a,s') |s]$.  Observe that
\begin{eqnarray*}
\begin{aligned}
\underset{\begin{subarray}{c} s \sim d^{\pi'} \\ a \sim \pi' \\ s' \sim P\end{subarray}}{\E} \left[ \delta_f (s,a,s') \right] & = \left\langle d^{\pi'}, \bar{\delta}_f^{\pi'} \right\rangle\\
& = \left\langle d^{\pi}, \bar{\delta}_f^{\pi'} \right\rangle + \left\langle d^{\pi'} - d^{\pi}, \bar{\delta}_f^{\pi'} \right\rangle
\end{aligned}
\end{eqnarray*}

This term is then straightforwardly bounded by applying H\"{o}lder's inequality; for any $p,q \in [1, \infty]$ such that $1/p + 1/q = 1$, we have
\begin{equation*}
\underset{\begin{subarray}{c} s \sim d^{\pi'} \\ a \sim \pi' \\ s' \sim P\end{subarray}}{\E} \left[ \delta_f (s,a,s') \right] \geq  \left\langle d^{\pi}, \bar{\delta}_f^{\pi'} \right\rangle - \left\| d^{\pi'} - d^{\pi}\right\|_p \left\|\bar{\delta}_f^{\pi'} \right\|_q. 
\end{equation*}

Particularly, we choose $p=1$ and $q = \infty$; however, we believe that this step is very interesting, and different choices for dealing with the inner product $\left\langle d^{\pi'} - d^{\pi}, \bar{\delta}_f^{\pi'} \right\rangle$ may lead to novel and useful bounds. 

With $\left\| d^{\pi'} - d^{\pi}\right\|_1 = 2 D_{TV} (d^{\pi'} || d^{\pi})$ and $\left\|\bar{\delta}_f^{\pi'} \right\|_{\infty} = \epsilon_f^{\pi'}$, the bound is almost obtained. The last step is to observe that, by the importance sampling identity,
\begin{eqnarray*}
\left\langle d^{\pi}, \bar{\delta}_f^{\pi'} \right\rangle &=& \underset{\begin{subarray}{c} s \sim d^{\pi} \\ a \sim \pi' \\ s' \sim P\end{subarray}}{\E} \left[ \delta_f (s,a,s') \right] \\
&=& \underset{\begin{subarray}{c} s \sim d^{\pi} \\ a \sim \pi \\ s' \sim P\end{subarray}}{\E} \left[\left( \frac{\pi'(a|s)}{\pi(a|s)} \right)\delta_f (s,a,s') \right]. 
\end{eqnarray*}

After grouping terms, the bound is obtained. 
\end{proof}

This lemma makes use of many ideas that have been explored before; for the special case of $f = V^{\pi}$, this strategy (after bounding $D_{TV} (d^{\pi'} || d^{\pi})$) leads directly to some of the policy improvement bounds previously obtained by Pirotta et~al. and Schulman et~al. The form given here is more general, however, because it allows for freedom in choosing $f$. (Although we do not report the results here, we note that this freedom allows for the derivation of analogous bounds involving Bellman errors of Q-function approximators, which is interesting and suggestive.)

\begin{remark}
It is reasonable to ask if there is a choice of $f$ which maximizes the lower bound here. This turns out to trivially be $f = V^{\pi'}$. Observe that $\E_{s' \sim P} \left[\delta_{V^{\pi'}}(s,a,s') | s,a\right] = A^{\pi'} (s,a)$. For all states, $\E_{a \sim \pi'} [A^{\pi'} (s,a)] = 0$ (by the definition of $A^{\pi'}$), thus $\bar{\delta}^{\pi'}_{V^{\pi'}} = 0$ and $\epsilon_{V^{\pi'}}^{\pi'} = 0$. Also, $L_{\pi, V^{\pi'}} (\pi') = -\E_{s \sim d^{\pi}, a \sim \pi} \left[A^{\pi'}(s,a)\right]$; from (\ref{jpif}) with $f = V^{\pi'}$, we can see that this exactly equals $J(\pi') - J(\pi)$. Thus, for $f = V^{\pi'}$, we recover an exact equality. While this is not practically useful to us (because, when we want to optimize a lower bound with respect to $\pi'$, it is too expensive to evaluate $V^{\pi'}$ for each candidate to be practical), it provides insight: the penalty coefficient on the divergence captures information about the mismatch between $f$ and $V^{\pi'}$. 
\end{remark}

Next, we are interested in bounding the divergence term, $\|d^{\pi'} - d^{\pi}\|_1$. We give the following lemma; to the best of our knowledge, this is a new result. 

\begin{lemma}\label{divergencebound}
The divergence between discounted future state visitation distributions, $\|d^{\pi'} - d^{\pi}\|_1$, is bounded by an average divergence of the policies $\pi'$ and $\pi$:
\begin{equation}
\|d^{\pi'} - d^{\pi}\|_1 \leq \frac{2\gamma}{1-\gamma} \underset{s \sim d^{\pi}}{\E} \left[ D_{TV} (\pi' || \pi)[s]\right], 
\end{equation}
where $D_{TV} (\pi'||\pi)[s] = (1/2) \sum_a |\pi'(a|s) - \pi(a|s)|$.  
\end{lemma}

\begin{proof}
First, using (\ref{dpidiff}), we obtain
\begin{eqnarray*}
\|d^{\pi'} - d^{\pi}\|_1 &=& \gamma \|\bar{G} \Delta d^{\pi}\|_1 \\
&\leq & \gamma \|\bar{G}\|_1 \|\Delta d^{\pi}\|_1.
\end{eqnarray*}
$\|\bar{G}\|_1$ is bounded by:
%
%\begin{eqnarray*}
%\|\bar{G}\|_1 &=& \|(I - \gamma P_{\pi'})^{-1}\|_1 \\
%&\leq& \sum_{t=0}^{\infty} \gamma^t \left\|P_{\pi'}\right\|_1^t \\
%&=& (1-\gamma)^{-1}%\frac{1}{1-\gamma}.
%\end{eqnarray*}
\begin{equation*}
\|\bar{G}\|_1 = \|(I - \gamma P_{\pi'})^{-1}\|_1 \leq \sum_{t=0}^{\infty} \gamma^t \left\|P_{\pi'}\right\|_1^t = (1-\gamma)^{-1}
\end{equation*}

To conclude the lemma, we bound $\|\Delta d^{\pi}\|_1$. 
\begin{eqnarray*}
\|\Delta d^{\pi}\|_1 &=& \sum_{s'} \left| \sum_s \Delta(s'|s) d^{\pi}(s) \right| \\
&\leq& \sum_{s,s'} \left| \Delta(s'|s)\right| d^{\pi}(s) \\
&=& \sum_{s,s'} \left| \sum_a P(s'|s,a) \left(\pi'(a|s) - \pi(a|s) \right)\right| d^{\pi}(s) \\
&\leq& \sum_{s,a,s'} P(s'|s,a) \left|\pi'(a|s) - \pi(a|s) \right| d^{\pi}(s) \\
&=& \sum_{s,a} \left|\pi'(a|s) - \pi(a|s) \right| d^{\pi}(s) \\
&=& 2 \underset{s \sim d^{\pi}}{\E} \left[ D_{TV} (\pi'||\pi)[s] \right].
\end{eqnarray*} 
\end{proof}

The new policy improvement bound follows immediately. 

\begin{theorem}\label{maintheorem}For any function $f : S \to \Real{}$ and any policies $\pi',\pi$, with $L_{\pi,f}(\pi')$ as defined in (\ref{surrogate}) and $\epsilon_f^{\pi'} \doteq \max_s \left| \E_{a \sim \pi', s'\sim P} [R(s,a,s') + \gamma f(s') - f(s)] \right|$, the following bound holds:
\begin{equation}
\begin{aligned}
&J(\pi') - J(\pi) \\
&\geq \frac{1}{1-\gamma}\left(L_{\pi,f} (\pi') - \frac{2\gamma \epsilon_f^{\pi'}}{1-\gamma} \underset{s \sim d^{\pi}}{\E} \left[ D_{TV} (\pi'||\pi)[s] \right] \right). \label{bound1}
\end{aligned}
\end{equation}
Furthermore, the bound is tight (when $\pi' = \pi$, the LHS and RHS are identically zero). 
\end{theorem}

\begin{proof} Begin with the bound of lemma \ref{policybound0} and bound the divergence $D_{TV} (d^{\pi'} || d^{\pi})$ by lemma \ref{divergencebound}.
\end{proof}

A few quick observations connect this result to prior work. Clearly, we could bound the expectation $\E_{s \sim d^{\pi}}\left[D_{TV} (\pi' || \pi)[s]\right]$ by $\max_s D_{TV} (\pi'||\pi)[s]$. Doing this, picking $f = V^{\pi}$, and bounding $\epsilon_{V^{\pi}}^{\pi'}$ to get a second factor of $\max_s D_{TV} (\pi'||\pi)[s]$, we recover (up to assumption-dependent factors) the policy improvement bounds given by Pirotta et~al. \yrcite{Pirotta2013} as Corollary 3.6, and by Schulman et~al. \yrcite{Schulman2006} as Theorem 1a. 

Because the choice of $f = V^{\pi}$ does allow for a nice simplification, we will give the bound with this choice as a corollary.

\begin{restatable}{corollary}{advantagebound} 
%\begin{corollary}
For any policies $\pi', \pi$, with $\epsilon^{\pi'} \doteq \max_s | \E_{a \sim \pi'} [A^{\pi} (s,a) ] |$, the following bound holds:
\begin{equation}
\begin{aligned}
&J(\pi') - J(\pi)\\
& \geq \frac{1}{1-\gamma} \underset{\begin{subarray}{c} s \sim d^{\pi} \\ a \sim \pi' \end{subarray}}{\E} \left[ A^{\pi} (s,a) - \frac{2\gamma \epsilon^{\pi'}}{1-\gamma}  D_{TV} (\pi'||\pi)[s] \right]. \label{bound2} %uh huh hunny
\end{aligned}
\end{equation}
\label{advantage-bound}
%\end{corollary}
\end{restatable}

\section{Easy Monotonic Policy Iteration}

As with the weaker versions of the bound, this bound allows us to generate a sequence of policies with non-degrading performance in any restricted class of policies, i.e. policies that are represented by neural networks or other function approximators. We'll use $\Pi_{\theta}$ to denote an arbitrary restricted class of policies, and understand that this usually means policies smoothly parametrized by some set of parameters $\theta$. Algorithm \ref{alg1} gives the general template for Easy Monotonic Policy Iteration (EMPI), which obtains monotonic improvements using (\ref{bound1}) (and one additional small step of bounding to remove $\pi'$ from the penalty coefficient). To see that EMPI is indeed monotonic, observe that $\pi_i$ is a feasible point of the optimization defining $\pi_{i+1}$, with objective value 0; this is a certificate that the objective at optimum is $\geq 0$. 

Although we do not specify how to solve the optimization problem, we note that the objective is differentiable with respect to the parameters of a candidate policy $\pi'$: as a result, a gradient-based method can be used. Furthermore, when we use neural network policies where the vector of $n$ parameters may take on any value in $\Real{n}$, the optimization is unconstrained. This is one sense in which we consider this algorithm `easy' to implement. 

\begin{algorithm}[tb]
   \caption{Easy Monotonic Policy Iteration: monotonic policy improvements in arbitrary policy classes}
   \label{alg1}
\begin{algorithmic}
   \STATE {\bfseries Input:} Initial policy $\pi_0 \in \Pi_{\theta}$, max number of iterations $N$, stopping tolerance $\alpha$
   %\REPEAT
	 \FOR{$i = 1,2,...,N$ or until $J(\pi_i) - J(\pi_{i-1}) \leq \alpha$} 
	 \STATE Choose function $f_i :S \to \Real{}$
	 \STATE $\pi_{i+1} \leftarrow \arg \underset{\pi' \in \Pi_{\theta}}{\max} L_{\pi_i,f_i} (\pi') - C \underset{s \sim d^{\pi_{i}}}{\E}\left[D_{TV} (\pi' || \pi_i)\right]$, \\
	where $C = \frac{2\gamma\epsilon}{1-\gamma}$, \\
	and $\epsilon = \max_{s,a} \left|\E_{s' \sim P} [R(s,a,s') + \gamma f_i(s') - f_i(s)] \right|$. 
   %\STATE $i \leftarrow i + 1$
	\ENDFOR
   %\UNTIL{$i > N$, or $J(\pi_i) - J(\pi_{i-1}) \leq \alpha$}
\end{algorithmic}
\end{algorithm}

Usually we would be interested in applying this method to problems with large state or unknown state spaces, where exact calculation of the objective function is not feasible. Because the objective is defined almost entirely in terms of expectations on policy $\pi_i$, however, we can use a sample-based approximation of it; this is the other sense in which this algorithm is `easy.' The only challenge is estimating $\epsilon$, where we might apply a worst-case bound (potentially making the policy step too conservative) or a reasonable heuristic bound (potentially permitting degraded policy performance). % We will address this issue later in a discussion of the results of Schulman et~al. 

\subsection{Implementing EMPI}

A practical form of the general algorithm is obtained by choosing a particular form for the reward-shaping functions $f_i$, possibly approximating terms in the objective, and replacing the expectations in the objective by sample estimates. We'll choose $f_i = V^{\pi_i}$, so that the base optimization problem becomes
\begin{equation*}
\begin{aligned}
% \frac{\pi'(a|s)}{\pi(a|s)}
\max_{\pi' \in \Pi_{\theta}} &&& \underset{\begin{subarray}{c} s \sim d^{\pi} \\ a \sim \pi' \end{subarray}}{\E} \left[ A^{\pi} (s,a) - \frac{2\gamma \epsilon}{1-\gamma} D_{TV} (\pi' || \pi)[s] \right], \\
&&& \text{where } \epsilon = \max_{s,a} |A^{\pi} (s,a)|. 
\end{aligned}
\end{equation*}
Suppose that we use an estimator of the advantage, $\hat{A}^{\pi}(s,a)$, instead of the true advantage $A^{\pi} (s,a)$. For example, if we have learned a neural network value function approximator $\hat{V}^{\pi}$, we may use $\hat{A}^{\pi} (s,a) = R(s,a,s') + \gamma \hat{V}^{\pi} (s') - \hat{V}^{\pi} (s)$; or perhaps we might use the generalized advantage estimator \cite{Schulman2015}. Observe that, because $\E_{a \sim \pi} [A^{\pi}(s,a)] = 0$, for every state $s$ we have
\begin{equation*}
\begin{aligned}
\underset{a \sim \pi'}{\E} &\left[A^{\pi}(s,a) \right] = \\ 
& \underset{a \sim \pi'}{\E} \left[\hat{A}^{\pi} (s,a) \right] - \underset{a \sim \pi}{\E} \left[\hat{A}^{\pi} (s,a) \right]\\
& + \sum_{a} \left( \pi'(a|s) - \pi(a|s) \right) \left( A^{\pi} (s,a) - \hat{A}^{\pi} (s,a) \right).
\end{aligned}
\end{equation*}
From this, we derive a bound:
\begin{equation*}
\begin{aligned}
\underset{a \sim \pi'}{\E} &\left[A^{\pi}(s,a) \right] \geq \\ 
& \underset{a \sim \pi'}{\E} \left[\hat{A}^{\pi} (s,a) \right] - \underset{a \sim \pi}{\E} \left[\hat{A}^{\pi} (s,a) \right]\\
& - 2 \max_a \left| A^{\pi} (s,a) - \hat{A}^{\pi} (s,a) \right| D_{TV} (\pi' || \pi) [s],
\end{aligned}
\end{equation*}
which gives us the following corollary to Theorem \ref{maintheorem}. 

\begin{corollary}[Policy Improvement Bound with Arbitrary Advantage Estimators] For any policies $\pi', \pi$, and any advantage estimator $\hat{A}^{\pi} : S \times A \to \Real{}$, with $c(s) \doteq \max_a | A^{\pi} (s,a) - \hat{A}^{\pi} (s,a) |$ and $\epsilon^{\pi'} \doteq \max_s | \E_{a \sim \pi'} [A^{\pi} (s,a) ] |$, the following bound holds:
\begin{equation}
\begin{aligned}
J(\pi') - J(\pi) &\\
\geq \frac{1}{1-\gamma} \underset{\begin{subarray}{c} s \sim d^{\pi} \\ a \sim \pi' \end{subarray}}{\E} &\Bigg[ \hat{A}^{\pi} (s,a) - \underset{\bar{a} \sim \pi}{\E}\left[ \hat{A}^{\pi} (s,\bar{a}) \right] \\
&  - 2 \left( c(s) + \frac{\gamma \epsilon^{\pi'}}{1-\gamma}\right)  D_{TV} (\pi'||\pi)[s] \Bigg]. \label{bound3} 
\end{aligned}
\end{equation}

Furthermore, the bound is tight (when $\pi' = \pi$, the LHS and RHS are identically zero). 
\label{arbitrary-advantage}
\end{corollary}

Corollary \ref{arbitrary-advantage} tells us that we are theoretically justified in using \textit{any} advantage estimators as long as we increase the penalty on the policy divergence appropriately. Also, we can see that if the estimator is high-quality ($|A^{\pi} (s,a) - \hat{A}^{\pi} (s,a)|$ small) we can take larger policy improvement steps, which makes sense. If the advantage estimator is poor, then we probably will not make much progress; this is reflected in the bound. 

The base optimization problem for EMPI with an advantage estimator, which we obtain by using (\ref{bound3}), dropping constants, and removing $\pi'$ from the penalty coefficient by bounding, is 
\begin{equation*}
\begin{aligned}
% \frac{\pi'(a|s)}{\pi(a|s)}
\max_{\pi' \in \Pi_{\theta}} &&& \underset{\begin{subarray}{c} s \sim d^{\pi} \\ a \sim \pi' \end{subarray}}{\E} \left[ \hat{A}^{\pi} (s,a) - C(s) D_{TV} (\pi' || \pi)[s] \right], \\
&&& \text{where } C(s) = 2\left(c(s) + \frac{\gamma \epsilon}{1-\gamma}\right), \\
&&& c(s) = \max_a \left|A^{\pi} (s,a) - \hat{A}^{\pi} (s,a) \right|, \\
&&& \text{and } \epsilon = \max_{s,a} |A^{\pi} (s,a)|. 
\end{aligned}
\end{equation*}

Now, we will put the objective in a form which can be sampled on policy $\pi$. First, we use the importance sampling identity so that the objective becomes
\begin{equation*}
\underset{\begin{subarray}{c} s \sim d^{\pi} \\ a \sim \pi \end{subarray}}{\E} \left[ \frac{\pi'(a|s)}{\pi(a|s)}\hat{A}^{\pi} (s,a) - C(s) D_{TV} (\pi' || \pi)[s] \right].
\end{equation*}
Next we rewrite the expectation in terms of trajectories:
\begin{equation}
\underset{\tau \sim \pi}{\E} \left[ \sum_{t=0}^{\infty} \gamma^t \left(\frac{\pi'(a_t|s_t)}{\pi(a_t|s_t)}\hat{A}^{\pi} (s_t,a_t) - C(s_t) D_{TV} (\pi' || \pi)[s_t] \right)\right].
\label{objective}
\end{equation}

After running an agent on policy $\pi$ to generate a set of sample trajectories, we can estimate (\ref{objective}) by averaging over the sample trajectories. The sample estimate of (\ref{objective}) then serves as the objective for the optimization step in EMPI. 

We do not give experimental results here, but plan to include them in a future version of this work.

\section{Implications for Trust Region Policy Optimization}

Schulman et~al. proposed Trust Region Policy Optimization, where at every iteration the policy is updated from $\pi$ to $\pi'$ by solving the following optimization problem:
\begin{equation}
% L_{\pi,V^{\pi}} (\pi')
\begin{aligned}
\pi' = \arg & \max_{\pi' \in \Pi_{\theta}} \underset{\begin{subarray}{c} s \sim d^{\pi} \\ a \sim \pi'\end{subarray}}{\E} \left[A^{\pi} (s,a) \right] \\
& \text{subject to} \;\; \underset{s \sim d^{\pi}}{\E} \left[D_{KL} (\pi || \pi') [s] \right] \leq \delta,
\end{aligned}
\label{TRPO}
\end{equation}
where $D_{KL} (\pi || \pi')[s] = \E_{a \sim \pi}\left[ \log \frac{\pi(a|s)}{\pi'(a|s)}\right]$, and the policy divergence limit $\delta$ is a hyperparameter of the algorithm.

The KL-divergence and the TV-divergence of arbitrary distributions $p,q$ are related by Pinsker's inequality, $D_{TV} (p || q) \leq \sqrt{D_{KL} (p || q)/2}$; combining this with Jensen's inequality, we obtain the following bound:
%
%\begin{eqnarray}
\begin{equation}
\begin{aligned}
\underset{s \sim d^{\pi}}{\E} \left[ D_{TV} (\pi'||\pi)[s] \right] \; &\leq \underset{s \sim d^{\pi}}{\E} \left[ \sqrt{\frac{1}{2} D_{KL} (\pi||\pi')[s]} \right] \\
&\leq \sqrt{\frac{1}{2} \underset{s \sim d^{\pi}}{\E} \left[ D_{KL} (\pi||\pi')[s] \right]}.
\label{pinskers}
\end{aligned}
\end{equation}
%\end{eqnarray}

%Theorem \ref{maintheorem}
By this bound and Corollary \ref{advantage-bound}, we have a result for the worst-case performance of TRPO.

\begin{corollary}[TRPO worst-case performance] \label{TRPOperf}
A lower bound on the policy performance difference between policies $\pi$ and $\pi'$, where $\pi'$ is given by (\ref{TRPO}) and $\pi \in \Pi_{\theta}$, is
\begin{equation}
J(\pi') - J(\pi) \geq \frac{-\sqrt{2\delta} \gamma \epsilon^{\pi'}}{(1-\gamma)^2},
\label{TRPOdiff}
\end{equation}
where $\epsilon^{\pi'} = \max_s | \E_{a \sim \pi'} [A^{\pi} (s,a) ] |$.
\end{corollary}

\begin{proof} $\pi$ is a feasible point of the optimization problem with objective value 0, so $\E_{s \sim d^{\pi}, a\sim \pi'} [A^{\pi} (s,a)] \geq 0$. The rest follows by  Corollary \ref{advantage-bound} and (\ref{pinskers}), noting that TRPO bounds the average KL-divergence by $\delta$.
\end{proof}

By (\ref{TRPOdiff}), we can see that TRPO is theoretically justified, in the sense that an appropriate selection of the hyper-parameter $\delta$ can guarantee a worst-case loss. 

%\begin{remark} The worst-case loss is proportional to $1/(1-\gamma)^2$, which may be quite steep for $\gamma$ close to 1. Picking $\delta$ to cancel this out could lead to unnecessarily conservative updates. But, we observe that $\delta$ and $\epsilon^{\pi'}$ are connected; for $\delta$ small, $\pi'$ and $\pi$ are generally close. The expected advantage function on-policy is identically zero, and $\epsilon^{\pi'} = \max_s | \E_{a \sim \pi'} [A^{\pi}(s,a)] |$; if $\pi'$ and $\pi$ are very close, this term may be small as well. Schulman et~al. found that TRPO tended to give monotonic improvements, which suggests that this is indeed the case, and heuristic approximations to $\epsilon$ are empirically justified. \end{remark}

\section{Discussion}

In this note, we derived a new policy improvement bound in which an average, rather than a maximum, policy divergence is penalized. We proposed Easy Monotonic Policy Iteration, an algorithm that exploits the bound to generate sequences of policies with guaranteed non-decreasing returns and which is easy to implement in a sample-based framework. We showed how to implement EMPI and theoretically justified the use of advantage estimators in the optimization in the inner loop. Lastly, we showed that our policy improvement bound gives a new theoretical foundation to Trust Region Policy Optimization, an algorithm for approximate monotonic policy improvements proposed by Schulman et~al. \yrcite{Schulman2006} which was shown to perform well empirically on a wide variety of tasks; here, we were able to bound the worst-case performance at each iteration of TRPO.

In a future version of this work, we will give experimental results from implementing EMPI on a range of reinforcement learning benchmarks, including high-dimensional domains like Atari.

\bibliography{bibliography}
\bibliographystyle{icml2016}

\end{document}